\documentclass{article}

\usepackage{amsmath}
\usepackage{amssymb}
\usepackage{amsthm}
\usepackage{eufrak}

\usepackage{tikz}

\usepackage[T1]{fontenc}
\DeclareFontFamily{T1}{pzc}{}
\DeclareFontShape{T1}{pzc}{m}{it}{<-> s * [1.10] pzcmi7t}{}
\DeclareMathAlphabet{\mathpzc}{T1}{pzc}{m}{it}
\newcommand{\logvar}[1]{\mathpzc{#1}}

\newcommand{\pro}{\mathbb{P}}
\newcommand{\pr}[1]{\mathbb{P}\!\left( #1 \right)}

\newtheorem{Theorem}{Theorem}

\begin{document}

\title{First-Order Bayesian Network Specifications \\ 
Capture the Complexity Class $\mathsf{PP}$}
\author{Fabio G. Cozman}

\maketitle

\section{Introduction}

The point of this note is to prove that a language is in the complexity class 
$\mathsf{PP}$ if and only if the strings of the language encode valid inferences 
in a Bayesian network defined using function-free first-order logic with equality.
Before this statement can be made precise, a number of definitions are needed. 
Section \ref{section:Background} summarizes the necessary background and
Section \ref{section:Characters} defines first-order Bayesian network specifications
and the complexity class $\mathsf{PP}$. 
Section \ref{section:CapturePP} states and proves the former captures the latter.

\section{Background}\label{section:Background}

We collect a number of definitions here \cite{Gradel2007,Libkin2012},
so as to fix our terminology and notation.

We consider input strings in the alphabet $\{0,1\}$; that is, a {\em string} is a sequence
of $0$s and $1$s. A {\em language} is a set of strings; a {\em complexity class} is 
a set of languages. 
A language is {\em decided} by a Turing machine if the machine accepts each string 
in the language, and rejects each string not in the language. The complexity class 
$\mathsf{NP}$ contains each languages that can be decided by a nondeterministic
Turing machine with a polynomial time bound. 

We focus on function-free first-order logic with equality (denoted by
$\mathsf{FFFO}$). That is, all formulas we contemplate are well-formed 
formulas of first-order logic with equality but without functions, containing
predicates, negation ($\neg$), conjunction ($\wedge$), disjunction ($\vee$),
implication ($\Rightarrow$), equivalence ($\Leftrightarrow$),   
 existential quatification ($\exists$)  and universal 
quantification ($\forall$).
The set of predicates is the {\em vocabulary}.

A formula $\phi$ in existential function-free second-order logic 
(denoted by $\mathsf{ESO}$) is a formula 
of the form $\exists r_1 \dots \exists r_m \phi'$, where $\phi'$ is a sentence
of $\mathsf{FFFO}$ containing predicates $r_1,\dots,r_m$. Such a sentence allows
existential quantification over the predicates themselves. Note that again we
have equality in the language (that is, the built-in predicate $=$ is always
available). 

For a given vocabulary, a structure $\mathfrak{A}$ is a pair consisting of a {\em domain} and
an {\em interpretation}. A domain is simply a set. An {\em interpretation} is
a truth assignment for every grounding of every predicate that is not 
existentially quantified. As an example, consider the following formula
of $\mathsf{ESO}$ as discussed by Gr\"adel \cite{Gradel2007}:
\[
\exists\mathsf{partition}: \forall \logvar{x}: \forall \logvar{y}:
\left( \mathsf{edge}(\logvar{x},\logvar{y}) \Rightarrow 
(\mathsf{partition}(\logvar{x}) \Leftrightarrow \neg \mathsf{partition}(\logvar{y}))
\right).
\]  
A domain is then a set that can be
taken as the set of nodes of an input graph. An interpretation is a truth
assignment for the $\mathsf{edge}$ predicate and can be taken as the set
of edges of the input graph. The formula is satisfied if and only if it is possible 
to partition the vertices into two subsets such that if a node is in one subset, 
it is not in the other. That is, the formula is satisfied if and only if the input
graph is bipartite.
 
We only consider finite vocabularies and finite domains in this note.
If a formula $\phi(\hat{\logvar{x}})$ has free logical variables $\hat{\logvar{x}}$, then
denote by $\mathfrak{A} \models \phi(\hat{a})$ the fact that formula 
$\phi(\hat{a})$ is true in structure $\mathfrak{A}$ when the logical 
variables $\hat{\logvar{x}}$ are replaced by elements of the domain $\hat{a}$. 
In this  case say that $\mathfrak{A}$ is a {\em model} of $\phi(\hat{a})$. 

Note that if $\phi(\hat{\logvar{x}})$ is a formula in $\mathsf{ESO}$
 as in the previous paragraphs, then its interpretations runs over the groundings
of the non-quantified predicates; that is, if $\phi$ contains predicates 
$r_1,\dots,r_m$ and $s_1,\dots,s_M$, but $r_1,\dots,r_m$ are all
existentially quantified, then a model for $\phi$ contains an intepretation
for $s_1,\dots,s_M$.

There is an {\em isomorphism} between structures $\mathfrak{A}_1$ and $\mathfrak{A}_2$
when there is a bijective mapping $g$ between the domains such that if
$r(a_1,\dots,a_k)$ is true in $\mathfrak{A}_1$, then $r(g(a_1),\dots,g(a_k))$ is true
in $\mathfrak{A}_2$, and moreover if $r(a_1,\dots,a_k)$ is true in $\mathfrak{A}_2$,
then $r(g^{1}(a_1),\dots,g^{-1}(a_k))$ is true in $\mathfrak{A}_1$ (where $g^{-1}$
denotes the inverse of $g$). A set of structures is {\em isomorphism-closed} if whenever
a structure is in the set, all structures that are isomorphic to it are also in the set.

We assume that every structure is given as a string, encoded as follows for a fixed
vocabulary \cite[Section 6.1]{Libkin2012}.
First, if the domain contains elements $a_1,\dots,a_n$, then the string begins with 
$n$ symbols $0$ followed by $1$. The vocabulary is fixed, so we take some order for 
the predicates, $r_1,\dots,r_m$. We then append, in this order, the encoding of
the interpretation of each predicate. Focus on predicate $r_i$ of arity $k$. 
To encode it with respect to a domain, we need to order the elements of the domain,
say $a_1 < a_2 < \dots < a_n$. This total ordering is assumed for now to be always 
available; it will be important later to check that the ordering itself can be defined. In any
case, with a total ordering we can enumerate lexicographically all $k$-tuples over the 
domain. Now suppose $\hat{a}_j$ is the $j$th tuple in this enumeration; then the $j$the 
bit of the encoding of $r_i$ is $1$ if $r(\hat{a}_j)$ is true in the given interpretation, and 
$0$ otherwise. Thus the encoding is a string containing 
$n+1+\sum_{i=1}^m n^{\mathrm{arity}(r_i)}$ symbols (either $0$ or $1$).

We can now state Fagin's theorem:

\begin{Theorem}
Let $\mathcal{S}$ be an isomorphism-closed set of finite structures of some 
non-empty finite vocabulary. Then $\mathcal{S}$ is in $\mathsf{NP}$ if and only if 
$\mathcal{S}$ is the class of finite models of a sentence in existential function-free
second-order logic.
\end{Theorem}

Denote by $\mathsf{CHECK}$ the problem of deciding whether an input structure is
a model of a fixed existential function-free second-order sentence.
Fagin theorem means first that $\mathsf{CHECK}$ is in $\mathsf{NP}$ (this is the easy
part of the theorem). Second, the theorem means that every  language that can be decided 
by a polynomial-time nonderministic Turing machine can be exactly encoded as the set 
of models for a sentence in existential second-order logic (this is the surprising part of 
the theorem). 
This implies that $\mathsf{CHECK}$ is $\mathsf{NP}$-hard, but the theorem is much
more elegant (because it says that there is no need for some polynomial processing outside 
of the specification provided by existential second-order logic). 

The significance of Fagin's theorem is that it offers a definition of $\mathsf{NP}$
that is not tied to any computational model; rather, it is tied to the expressivity
of the language that is used to specify problems. Any language  
that can be decided by a polynomial nondeterministic Turing machine can 
equivalently be
be decided using first-order logic with some added quantification over predicates. 

\section{First-order Bayesian network specifications and the complexity class $\mathsf{PP}$}
\label{section:Characters}

We start by defining our two main characters: on one side we have Bayesian networks
that are specified using $\mathsf{FFFO}$; on the other side we have the complexity
class $\mathsf{PP}$. 

It will now be convenient to view each grounded predicate $r(\hat{a})$ as a random
variable once we have a fixed vocabulary and domain. So, given a domain $\mathcal{D}$, 
we understand $r(\hat{a})$ as a function over all possible interpretations of the
vocabulary, so that $r(\hat{a})(\mathbb{I})$ yields $1$ if $r(\hat{a})$ is true in interpretation
$\mathbb{I}$, and $0$ otherwise.

\subsection{First-order Bayesian network specifications}

A {\em first-order Bayesian network specification} is a directed graph where each
node is a predicate, and where each root node $r$ is associated with a probabilistic
assessment 
\[
\pr{r(\hat{\logvar{x}})=1}=\alpha,
\]
while each non-root node $s(\hat{\logvar{x}})$ is associated with a formula (called the
{\em definition} of $s$)
\[
s(\hat{x}) \Leftrightarrow \phi(\hat{\logvar{x}}),
\]
where $\phi(\hat{\logvar{x}})$ is a formula in $\mathsf{FFFO}$ with free variables $\hat{\logvar{x}}$. 

Given a domain, a first-order Bayesian network specification can be grounded into
a unique Bayesian network. This is done:
\begin{enumerate}
\item by producing every grounding of the predicates, 
\item by associating with each grounding $r(\hat{a})$ of a root predicate the
grounded assessment $\pr{r(\hat{a})=1}=\alpha$;
\item by associating with each grounding $s(\hat{a})$ of a non-root predicate the
grounded definition $s(\hat{a}) \Leftrightarrow \phi(\hat{a})$;
\item finally, by drawing a graph where each node is a grounded predicate and where
there is an edge into each grounded non-root predicate $s(\hat{a})$ from each
grounding of a predicate that appears in the grounded definition of $s(\hat{a})$. 
\end{enumerate}

Consider, as an example, the following model of asymmetric friendship, where
an individual is always a friend of herself, and where two
individuals are friends if they are both fans (of some writer, say) or if there is some
``other'' reason for it:
\begin{eqnarray}
\label{equation:Friends}
\pr{\mathsf{fan}(\logvar{x})} & = & 0.2, \nonumber \\
\pr{\mathsf{friends}(\logvar{x},\logvar{y})} & \Leftrightarrow & (\logvar{x} = \logvar{y}) \vee \nonumber \\
& & (\mathsf{fan}(\logvar{x}) \wedge \mathsf{fan}(\logvar{y})) \vee \\
& & \mathsf{other}(\logvar{x},\logvar{y}), \nonumber \\
\pr{\mathsf{other}(\logvar{x},\logvar{y})} & = & 0.1. \nonumber
\end{eqnarray}
Suppose we have domain $\mathcal{D}=\{a,b,c\}$. Figure \ref{figure:Friends} depicts
 the Bayesian
network generated by $\mathcal{D}$ and Expression (\ref{equation:Friends}).

\begin{figure}
\begin{center}
\begin{tikzpicture}
\node[rectangle,rounded corners,draw,fill=yellow] (fa) at (4,2.5) {$\mathsf{fan}(a)$};
\node[rectangle,rounded corners,draw,fill=yellow] (fb) at (6,2.5) {$\mathsf{fan}(b)$};
\node[rectangle,rounded corners,draw,fill=yellow] (fc) at (8,2.5) {$\mathsf{fan}(c)$};
\node[rectangle,rounded corners,draw,fill=yellow] (frab) at (1,1) {$\mathsf{friends}(a,b)$};
\node[rectangle,rounded corners,draw,fill=yellow] (frac) at (3,1) {$\mathsf{friends}(a,c)$};
\node[rectangle,rounded corners,draw,fill=yellow] (frba) at (5,1) {$\mathsf{friends}(b,a)$};
\node[rectangle,rounded corners,draw,fill=yellow] (frbc) at (7,1) {$\mathsf{friends}(b,c)$};
\node[rectangle,rounded corners,draw,fill=yellow] (frca) at (9,1) {$\mathsf{friends}(c,a)$};
\node[rectangle,rounded corners,draw,fill=yellow] (frcb) at (11,1) {$\mathsf{friends}(c,b)$};
\node[rectangle,rounded corners,draw,fill=yellow] (oab) at (1,0) {$\mathsf{other}(a,b)$};
\node[rectangle,rounded corners,draw,fill=yellow] (oac) at (3,0) {$\mathsf{other}(a,c)$};
\node[rectangle,rounded corners,draw,fill=yellow] (oba) at (5,0) {$\mathsf{other}(b,a)$};
\node[rectangle,rounded corners,draw,fill=yellow] (obc) at (7,0) {$\mathsf{other}(b,c)$};
\node[rectangle,rounded corners,draw,fill=yellow] (oca) at (9,0) {$\mathsf{other}(c,a)$};
\node[rectangle,rounded corners,draw,fill=yellow] (ocb) at (11,0) {$\mathsf{other}(c,b)$};
\draw[->,>=latex,thick] (fa)--(frab);
\draw[->,>=latex,thick] (fa)--(frac);
\draw[->,>=latex,thick] (fa)--(frba);
\draw[->,>=latex,thick] (fa)--(frca);
\draw[->,>=latex,thick] (fb)--(frba);
\draw[->,>=latex,thick] (fb)--(frbc);
\draw[->,>=latex,thick] (fb)--(frab);
\draw[->,>=latex,thick] (fb)--(frcb);
\draw[->,>=latex,thick] (fc)--(frca);
\draw[->,>=latex,thick] (fc)--(frcb);
\draw[->,>=latex,thick] (fc)--(frac);
\draw[->,>=latex,thick] (fc)--(frbc);
\draw[->,>=latex,thick] (oab)--(frab);
\draw[->,>=latex,thick] (oac)--(frac);
\draw[->,>=latex,thick] (oba)--(frba);
\draw[->,>=latex,thick] (obc)--(frbc);
\draw[->,>=latex,thick] (oca)--(frca);
\draw[->,>=latex,thick] (ocb)--(frcb);
\node[rectangle,rounded corners,draw,fill=yellow] (fraa) at (4,3.5) {$\mathsf{friends}(a,a)$};
\node[rectangle,rounded corners,draw,fill=yellow] (oaa) at (4,4.5) {$\mathsf{other}(a,a)$};
\draw[->,>=latex,thick] (fa)--(fraa);
\draw[->,>=latex,thick] (oaa)--(fraa);
\node[rectangle,rounded corners,draw,fill=yellow] (frbb) at (6,3.5) {$\mathsf{friends}(b,b)$};
\node[rectangle,rounded corners,draw,fill=yellow] (obb) at (6,4.5) {$\mathsf{other}(b,b)$};
\draw[->,>=latex,thick] (fb)--(frbb);
\draw[->,>=latex,thick] (obb)--(frbb);
\node[rectangle,rounded corners,draw,fill=yellow] (frcc) at (8,3.5) {$\mathsf{friends}(c,c)$};
\node[rectangle,rounded corners,draw,fill=yellow] (occ) at (8,4.5) {$\mathsf{other}(c,c)$};
\draw[->,>=latex,thick] (fc)--(frcc);
\draw[->,>=latex,thick] (occ)--(frcc);
\end{tikzpicture}
\end{center}
\caption{The Bayesian network generated by Expression (\ref{equation:Friends}) and domain
$\mathcal{D}=\{a,b,c\}$.}
\label{figure:Friends}
\end{figure}
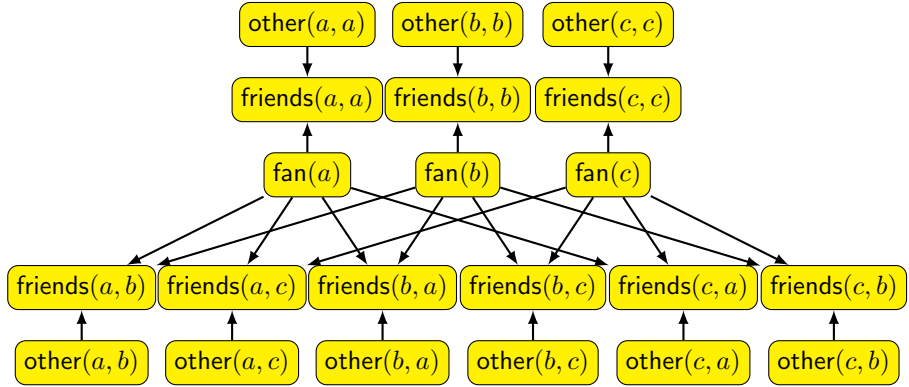

For a given Bayesian network specification $\tau$ and a domain $\mathcal{D}$, denote
by $\mathbb{B}(\tau,\mathcal{D})$ the Bayesian network obtained by grouning $\tau$
with respect to $\mathcal{D}$. 
The set of all first-order Bayesian network specifications is denoted by $\mathcal{B}(\mathsf{FFFO})$.

\subsection{Probabilistic Turing machines and the complexity class $\mathsf{PP}$}

If a Turing machine is such that, whenever its transition function maps to a
non-singleton set, the transition is selected with uniform probability
within that set, then the Turing machine is a {\em probabilistic Turing machine}. 
The complexity class $\mathsf{PP}$ is the set of languages that are decided
by a probabilistic Turing machine in polynomial time, with an error probability
strictly less than $1/2$ for all input strings. 

Intuitively, $\mathsf{PP}$ represents the complexity of computing probabilities
for a phenomenon that can be simulated by a polynomial probabilistic Turing machine.

This complexity class can be equivalently defined as follows: 
a language is in $\mathsf{PP}$ if and only if there is a polynomial nondeterministic
Turing machine such that a string is  in the language if and only if  more than 
half of the computation paths of the machine end in the accepting state when the
string is the input. 
We can imagine that there is a special class of nondeterministic Turing
machines that, given an input, not only accept it or not, but actually write in some
special tape whether that input is accepted in the majority of computation paths. 
Such a special machine could then be used directly to decide a language in
$\mathsf{PP}$. 

\section{$\mathcal{B}(\mathsf{FFFO})$ captures $\mathsf{PP}$}
\label{section:CapturePP}

Given a first-order Bayesian network specification and a domain, an {\em evidence piece}
$\mathbf{E}$ is a partial interpretation; that is, an evidence piece assigns a truth value for
some groundings of predicates. 

We encode a pair domain/evidence $(\mathcal{D},\mathbf{E})$ using the same strategy 
used before to  encode a structure; however, we must take into account the fact that
a particular grounding of a predicate can be either assigned true or false or be left
without assignment. So we use a pair of symbols in $\{0,1\}$ to encode each
grounding; we assume that $00$ 
means ``false'' and $11$ means ``true'', while say $01$ means lack of assignment. 

Say there is an {\em isomorphism} between pairs $(\mathcal{D}_1,\mathbf{E}_1)$
and $(\mathcal{D}_2,\mathbf{E}_2)$ 
when there is a bijective mapping $g$ between the domains such that if
$r(a_1,\dots,a_k)$ is true in $\mathbf{E}_1$, then $r(g(a_1),\dots,g(a_k))$ is true
in $\mathbf{E}_2$, and moreover if $r(a_1,\dots,a_k)$ is true in $\mathbf{E}_2$,
then $r(g^{1}(a_1),\dots,g^{-1}(a_k))$ is true in $\mathbf{E}_1$ (where again $g^{-1}$
denotes the inverse of $g$). A set of pairs domain/evidence is {\em isomorphism-closed} 
if whenever a pair is in the set, all pairs that are isomorphic to it are also in the set.

Suppose a set of pairs domain/evidence 
is given with respect to a fixed vocabulary $\sigma$. Once encoded, these pairs 
form a language $\mathcal{L}$ that can for instance belong to $\mathsf{NP}$ or to $\mathsf{PP}$. 
One can imagine building a Bayesian network specification $\tau$ on an extended
vocabulary consisting of $\sigma$ plus some additional predicates, so as to
decide this language $\mathcal{L}$ of domain/evidence pairs. For a given input pair
$(\mathcal{D},\mathbf{E})$, the Bayesian network specification and the domain
lead to a Bayesian network $\mathbb{B}(\tau,\mathcal{D})$; this 
network can be used  to compute the probability of some groundings, and that
probabiility in turn can be   
used to accept/reject the input. This is the sort of strategy we pursue.

The point is that we must determine some prescription by which, given a Bayesian 
network and an evidence piece, one can generate an actual decision so as to accept/reject 
the input pair domain/evidence. Suppose we take the following strategy. Assume that
in the extended vocabulary of $\tau$ there are two sets of distinguished auxiliary 
predicates $A_1,\dots,A_{m'}$ and $B_1,\dots,B_{m''}$ that are not in $\sigma$. 
We can use the Bayesian network $\mathbb{B}(\tau,\mathcal{D})$ to
compute the probability $\pr{\mathbf{A}|\mathbf{B},\mathbf{E}}$  where 
$\mathbf{A}$ and $\mathbf{B}$ are interpretations of $A_1,\dots,A_{m'}$ and
$B_1,\dots,B_{m''}$ respectively. And then we might accept/reject the input on
the basis of $\pr{\mathbf{A}|\mathbf{B},\mathbf{E}}$. However, we cannot
specify particular intepretations $\mathbf{A}$ and $\mathbf{B}$ as the related
predicates are not in the vocabulary $\sigma$. Thus the sensible strategy is
to fix attention to some selected pair of intepretations; 
we simply take the interpretations that assign true to every grounding.

In short: use the Bayesian network $\mathbb{B}(\tau,\mathcal{D})$
to determine whether or not
$\pr{\mathbf{A}|\mathbf{B},\mathbf{E}}>1/2$, where
$\mathbf{A}$ assigns true to every grounding of $A_1,\dots,A_{m'}$,
and 
$\mathbf{B}$ assigns true to every grounding of $B_1,\dots,B_{m''}$.
If this inequality is satisfied, the input pair is {\em accepted}; if not,
the input pair is {\em rejected}.

We refer to $A_1,\dots,A_{m'}$ as the {\em conditioned predicates}, and
to $B_1,\dots,B_{m''}$ as the {\em conditioning predicates}.

Here is the main result:

\begin{Theorem}\label{theorem:CapturePP}
Let $\mathcal{S}$ be an isomorphism-closed set of pairs domain/evidence of some
non-empty finite vocabulary, where all domains are finite.
Then $\mathcal{S}$ is in $\mathsf{PP}$ if and only if $\mathcal{S}$ is the class
of domain/evidence pairs that are accepted by a fixed first-order Bayesian network 
specification with fixed conditioned and conditioning   predicates. 
\end{Theorem}
\begin{proof}
First, if $\mathcal{S}$ is a class of domain/query pairs that are accepted by a
fixed first-order Bayesian network specification, they can be decided by
a polynomial time probability Turing machine. To see that, note that we
can build a nondeterministic Turing machine that guesses the truth value of
all groundings that do not appear in the query (that is, not in 
$\mathbf{A}\cup\mathbf{B}\cup\mathbf{E}$), and then verify whether the resulting
complete interpretation is a model of the first-order Bayesian network specification
(as model checking of a fixed first-order sentence is in $\mathsf{P}$ \cite{Libkin2012}).

To prove the other direction,
we must adapt the proof of Fagin's theorem as described by Gr\"adel \cite{Gradel2007},
along the same lines as the proof of Theorem 1 by Saluja et al.\ \cite{Saluja95}.
So, suppose that $\mathcal{L}$ is a language decided by some probabilistic
Turing machine. So equivalently there is a nondeterministic Turing machine 
that determines whether the majority of its computation paths accept an input,
and accepts/rejects the input accordingly. By the mentioned proof of Fagin's 
theorem, there is a first-order sentence $\phi'$ with vocabulary consisting of the vocabulary 
of the input plus additional auxiliary predicates, such that each interpretation of
this joint vocabulary is a model of the sentence if it is encodes a computation
path of the Turing machine, as long as there is an available additional predicate
that is guaranteed to be a linear order on the domain. Denote by $A$ the
zero arity predicate with associated definition $A \Leftrightarrow \phi'$.
Suppose a linear order
is indeed available; then by creating a first-order Bayesian network specification
where all groundings are associated with probability $1/2$, and where a
non-root node is associated with the sentence in the proof of Fagin's theorem, 
we have that the probability of the query is larger than $1/2$ iff the
majority of computation paths accept. The challenge is to encode a linear
order. To do so, introduce a new predicate $<$ and  the first-order sentence $\phi''$
that forces $<$ to be a total order, and a zero arity predicate $B$ that is 
associated with definition  $B \Leftrightarrow \phi''$. Now an input domain/pair 
$(\mathcal{D},\mathbf{E})$ is accepted by
the majority of computation paths in the Turing machine if and only if
we have $\pr{A|B,\mathbf{E}}>1/2$. Note that there are actually $n!$
linear orders that satisfy $B$, but for each one of these linear orders we
have the same assignments for all other predicates, hence the ratio between
accepting computations and all computations is as desired. 
\end{proof}

We might picture this as follows. There is always a Turing machine $\mathbb{TM}$
and a corresponding triple $(\tau,A,B)$ such that for any
pair $(\mathcal{D},\mathbf{E})$, we have
\[
(\mathcal{D},\mathbf{E}) \mbox{ as input to } \mathbb{TM} \mbox{ with output given by }  
\pr{\mathbb{TM} \mbox{ accepts } (\mathcal{D},\mathbf{E})} > 1/2,
\]
if and only if 
\[
(\mathcal{D},\mathbf{E}) \mbox{ as ``input'' to }  (\tau,A,B) \mbox{ with ``output'' given by } 
\pro_{\tau,\mathcal{D}}(A|B,\mathbf{E})>1/2,
\]
where $\pro_{\tau,\mathcal{D}}(A|B,\mathbf{E})$ denotes   probability
with respect to $\mathbb{B}(\tau,\mathcal{D})$. (Of course, there is no need
to use only zero-arity predicates $A$ and $B$, as Theorem \ref{theorem:CapturePP}
allows for sets of predicates.)

Note that the same result could be proved if every evidence piece was taken 
to be a complete interpretation for the vocabulary $\sigma$. In that case we could
directly speak of structures as inputs, and then the result would more closely
mirror Fagin's theorem. However it is very appropriate, and entirely in line
with practical use, to take the inputs to a Bayesian network as the 
groundings of a partially observed interpretation. Hence we have preferred
to present our main result as stated in Theorem \ref{theorem:CapturePP}.


\begin{thebibliography}{10}

\bibitem{Gradel2007}
Erich Gr\"adel.
\newblock Finite model theory and descriptive complexity.
\newblock In Gr\"adel, E., Kolaitis, P .G., Libkin, L., Marx, M., Spencer, J., Vardi, 
M.Y., Venema, Y., Weinstein, S., editors,
{\em Finite Model Theory and Its Applications}, pp. 125--230, Springer 2007.

\bibitem{Libkin2012}
Leonid Libkin.
\newblock {\em Elements of Finite Model Theory}, Springer, 2012.

\bibitem{Saluja95}
Sanjeev Saluja, K. V. Subrahmanyam and Madhukar N. Thakur.
\newblock Descriptive complexity of $\#\mathsf{P}$ functions.
\newblock {\em Journal of Computer and System Sciences}, 50:493--505, 1995.

\end{thebibliography}
\end{document}